\documentclass{article}

\usepackage{microtype}
\usepackage{graphicx}
\usepackage{subcaption}
\usepackage{booktabs} % for professional tables
\usepackage{makecell}
\usepackage{algpseudocode}
\usepackage{booktabs}
\usepackage{hyperref}
\usepackage[normalem]{ulem}

\usepackage[accepted]{icml2026}

\usepackage{amsmath}
\usepackage{amssymb}
\usepackage{mathtools}
\usepackage{amsthm}
\usepackage{xcolor}
\usepackage{fontawesome5}

\usepackage[capitalize,noabbrev]{cleveref}

\theoremstyle{plain}
\newtheorem{theorem}{Theorem}[section]
\newtheorem{proposition}[theorem]{Proposition}

\theoremstyle{definition}

\theoremstyle{remark}

\usepackage[textsize=tiny]{todonotes}

\definecolor{linkblue}{RGB}{25, 55, 160}

\newcommand{\projectlink}[3]{%
  \href{#1}{\textcolor{linkblue}{#2\ #3}}%
}

\begin{document}

\twocolumn[
\icmltitle{OmniRAG-Agent: Agentic Omnimodal Reasoning for Low-Resource Long Audio-Video Question Answering}

\vspace{0.1em}
\centerline{\textbf{Yifan Zhu}$^{1}$ \quad \textbf{Xinyu Mu}$^{5}$ \quad \textbf{Tao Feng}$^{3}$ \quad \textbf{Zhonghong Ou}$^{1}$}
\vspace{0.05em}
\centerline{\textbf{Yuning Gong}$^{4}$ \quad \textbf{Haoran Luo}$^{2,\dagger}$}

\vspace{0.2em}
\centerline{$^{1}$Beijing University of Posts and Telecommunications \quad $^{2}$Nanyang Technological University \quad
$^{3}$Tsinghua University}
\centerline{
$^{4}$National University of Singapore \quad
$^{5}$Beijing Information Science and Technology University}

\vspace{0.1em}
\centerline{\texttt{yifan\_zhu@bupt.edu.cn, 2022011260@bistu.edu.cn, haoran.luo@ieee.org}}

\vspace{0.05em}
\centerline{
\projectlink{https://jackefn.github.io/OmniRAG-Agent-Home/}{\faHome}{Homepage}
\quad
\projectlink{https://jackefn.github.io/OmniRAG-Agent-Home/demo.html}{\faPlayCircle}{Demo}
\quad
\projectlink{https://github.com/jackefn/OmniRAG-Agent.git}{\faGithub}{GitHub}
\quad
\projectlink{https://huggingface.co/datasets/JackMuX3Y/OmniRAG-Agent}{\faDatabase}{Dataset}
}

\vskip 0.15in
]
\makeatletter
\renewcommand{\thefootnote}{\fnsymbol{footnote}}
\setcounter{footnote}{2}
\makeatother
\footnotetext{Corresponding author.}
% this must go after the closing bracket ] following \twocolumn[ ...

% This command actually creates the footnote in the first column listing the
% affiliations and the copyright notice. The command takes one argument, which
% is text to display at the start of the footnote. The \icmlEqualContribution
% command is standard text for equal contribution. Remove it (just {}) if you
% do not need this facility.

% Use ONE of the following lines. DO NOT remove the command.
% If you have no special notice, KEEP empty braces:
% \printAffiliationsAndNotice{}  % no special notice (required even if empty)
% Or, if applicable, use the standard equal contribution text:
% \printAffiliationsAndNotice{\icmlEqualContribution}

\begin{abstract}
  Long-horizon omnimodal question answering answers questions by reasoning over text, images, audio, and video. Despite recent progress on OmniLLMs, low-resource long audio-video QA still suffers from costly dense encoding, weak fine-grained retrieval, limited proactive planning, and no clear end-to-end optimization.To address these issues, we propose OmniRAG-Agent, an agentic omnimodal QA method for budgeted long audio-video reasoning. It builds an image–audio retrieval-augmented generation module that lets an OmniLLM fetch short, relevant frames and audio snippets from external banks. Moreover, it uses an agent loop that plans, calls tools across turns, and merges retrieved evidence to answer complex queries. Furthermore, we apply group relative policy optimization to jointly improve tool use and answer quality over time. Experiments on OmniVideoBench, WorldSense, and Daily-Omni show that OmniRAG-Agent consistently outperforms prior methods under low-resource settings and achieves strong results, with ablations validating each component.
\end{abstract}

\section{Introduction}

\begin{figure}[htbp]
    \centering 
    \includegraphics[width=\columnwidth]{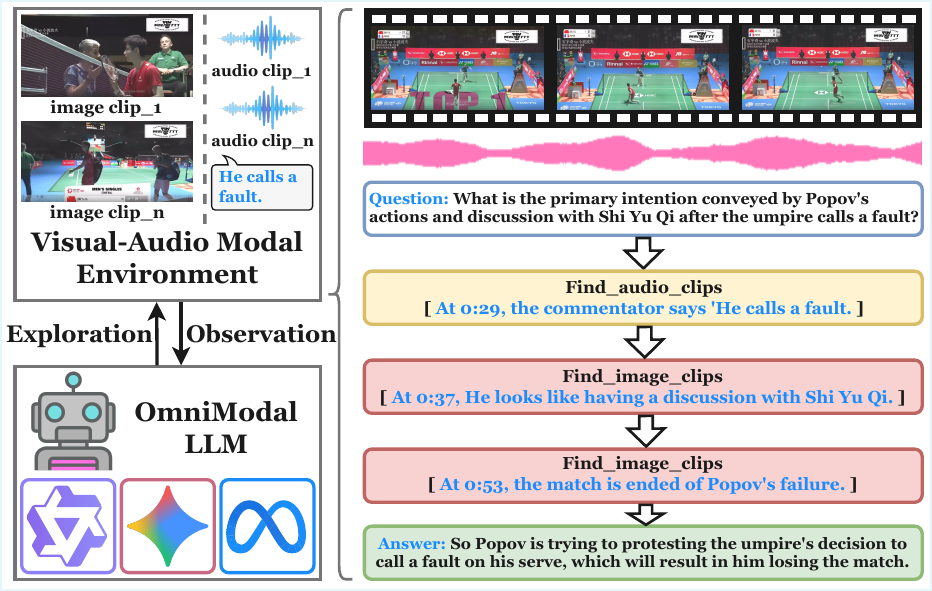}
    \caption{An example of OmniRAG-Agent interacting with a visual–audio modal environment. The agent answers a long-horizon question by iteratively retrieving relevant audio clips and image clips from external banks and reasoning step by step.}
    \label{fig:fig_1}
\end{figure}

Omnimodal large language models (OmniLLMs)~\cite{xu2025qwen2,zhao2025humanomni,comanici2025gemini} unify perception and generation over text, images, audio, and video, and have become a fundamental capability for embodied agents, intelligent assistants, and multimodal foundation models~\cite{zhao2025r1,long2025seeing}. In real-world applications, long audio and video inputs are often needed because the clues for a question are usually scattered across several minutes and must be pieced together from both what is seen and what is heard to produce a reliable answer~\cite{kulkarni2025avatar,team2025longcat}. However, most existing OmniLLMs and omnimodal QA systems are not built for long audio and video under tight budgets: encoding frames and audio segments densely quickly becomes too expensive in compute and memory~\cite{jiang2025specific}, which makes low-resource long-horizon audio-video reasoning an increasingly important research direction.

With the emergence and rapid progress of OmniLLMs, two main types of approaches have appeared for long-horizon audio-video question answering. On the one hand, API-based Omni systems call multiple specialized models to analyze video and audio step by step, and then combine the intermediate results to answer complex questions~\cite{tao2025omniagent}. On the other hand, fine-tuning-based Omni systems often rely on training strategies such as reinforcement learning~\cite{rafailov2023direct,schulman2017proximal,yu2025dapo} to optimize an OmniLLM, making it better at spotting small but important details in audio and video~\cite{zhong2025omni}. Overall, both methods can significantly improve OmniLLMs for long-horizon audio-video reasoning.

\begin{figure*}[htbp]
  \centering
  \includegraphics[width=1.0\textwidth]{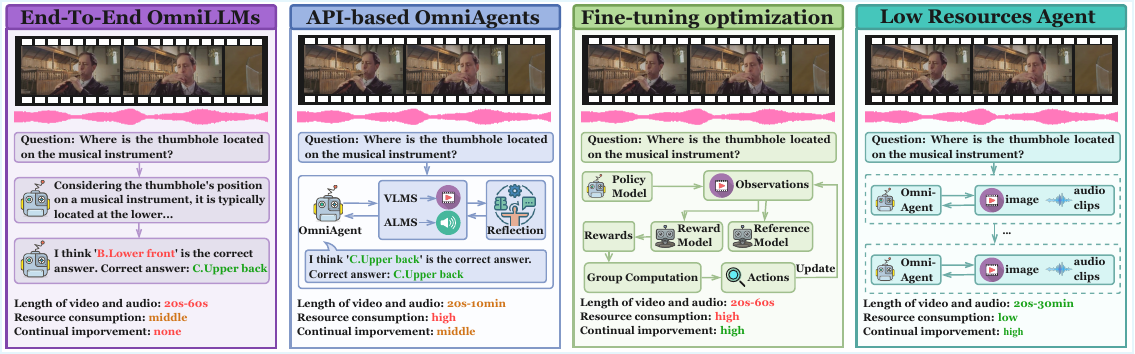}
  \caption{Comparison of different approaches for long-horizon audio-video QA under low-resource constraints: end-to-end OmniLLMs, API-based OmniAgents, LLM fine-tuning optimization, and our low-resource agent training framework with image–audio clip retrieval.}
  \label{fig:fig_2}
\end{figure*}

However, three main challenges remain. \textbf{(1) Limited fine-grained evidence retrieval under tight budgets.} For API-based Omni systems~\cite{xu2025qwen3omnitechnicalreport}, this often means running large multimodal models over raw audio-video for broad analysis, which quickly becomes expensive and still does not reliably find the exact short segments that matter. \textbf{(2) Limited proactive planning in single-model reasoning.} For fine-tuning-based Omni systems, training is usually done on short clips due to compute constraints, so the resulting model does not naturally generalize to long-horizon inputs where broader coverage and smarter selection are needed. \textbf{(3) Lack of effective end-to-end training and optimization mechanisms.} Most current pipelines optimize retrieval and answering separately, so they cannot directly learn better evidence selection and better final answers from overall performance, making it hard to steadily improve long-horizon audio-video reasoning over time.

To address these challenges, we propose OmniRAG-Agent, a new heuristic omnimodal QA approach that can interact with an image bank and an audio bank, as shown in Figure~\ref{fig:fig_1}. First, we build an image-audio retrieval-augmented generation (RAG) framework~\cite{luo2025kbqa,xiammed,wangretrieval}, where an OmniLLM can call retrieval tools to obtain fine-grained evidence and strengthen omnimodal reasoning in real time. Moreover, we introduce an agentic reasoning framework~\cite{li2025iterative,wengwe,leispider} that lets the OmniLLM think on its own and make multiple tool calls across turns, which helps solve complex questions that need information from different audio-visual angles. Furthermore, we apply a reinforcement learning method, group relative policy optimization (GRPO)~\cite{zheng2025group}, to keep improving the OmniLLM over time.

We conduct extensive experiments on three long-horizon omnimodal QA benchmarks, OmniVideoBench~\cite{li2025omnivideobench}, WorldSense~\cite{hong2025worldsense}, and Daily-Omni~\cite{zhou2025daily}. Experimental results demonstrate that our approach consistently outperforms existing methods (as illustrated in Figure~\ref{fig:fig_2} and Table~\ref{tab:method_comparison}) under low-resource settings, achieving state-of-the-art performance across multiple evaluation metrics. Further analysis and ablation studies confirm the effectiveness of each proposed component, highlighting the robustness and scalability of our framework for real-world long-duration omnimodal reasoning tasks. In addition, our RAG method and multi-round retrieval mechanism support multiple open-source models, including Qwen2.5-Omni~\cite{xu2025qwen2}, Qwen3-Omni~\cite{xu2025qwen3omnitechnicalreport}, making it a promising plug-and-play solution.

\section{Related Work}

\textbf{Omnimodal large language models.} Research on omnimodal models extends earlier text-only LLMs~\cite{achiam2023gpt} toward richer multimodal understanding. While VLMs~\cite{yang2023dawn,bai2025qwen3vltechnicalreport} mainly combine vision and language, omnimodal large language models (OmniLLMs)~\cite{hu2024minicpm,li2025baichuan,xing2025echoink} further incorporate audio to support unified perception and generation. For example, Qwen3-Omni~\cite{xu2025qwen3omnitechnicalreport} enables end-to-end speech input and output, and HumanOmni~\cite{zhao2025humanomni} focuses on human-centric audio-visual understanding. Meanwhile, closed-source models such as Gemini ~\cite{team2023gemini} have also shown strong audio-video capabilities for real-world interaction.

\begin{table*}[ht]
\caption{Comparison of methods for long-horizon omnimodal QA. These five dimensions indicate a method’s ability to solve long-horizon video-audio question-answering tasks.}
\centering
\small
\begin{tabular}{l|c|c|c|c|c}
\Xhline{1pt}
\textbf{Method} & \textbf{Modality} & \textbf{Max Input Length} & \textbf{Agentic} & \textbf{Deployment Cost} & \textbf{Continual Learning} \\
\hline
Qwen2.5-Omni & Audio + Vision & 60s & No & 60GB VRAM & None \\
\hline
VideoLLama2.1 & Audio + Vision & 60s & No & 60GB VRAM & None \\
\hline
Qwen3-VL & Vision Only & 5min & No & 40GB VRAM & None \\
\hline
XGC-AVis & Audio + Vision & 10min & Yes & Higher token budget & Low \\
\hline
OmniAgent & Audio + Vision & 30min & Yes & Higher token budget & Low \\
\hline
\textbf{OmniRAG-Agent (Ours)} & Audio + Vision & 30min & Yes & 15GB VRAM & High \\
\Xhline{1pt}
\end{tabular}
\label{tab:method_comparison}
\end{table*}

\textbf{VLM-powered agents.} In recent years, with the progress of VLMs, many studies have explored agentic approaches that allow VLMs to actively browse videos and gather key information for question answering~\cite{wang2025videotree,zhang2024simple,jeoung2024adaptive}. Meanwhile, other methods focus on designing multiple tool modules, enabling VLMs to select and call appropriate tools as needed~\cite{fan2024videoagent,min2024morevqa,liu2025videomind}. In addition, some works rely on predefined workflows to guide VLMs through structured exploration, which helps improve long-video understanding in complex real-world scenarios~\cite{yuan2025videodeepresearch,shu2025audio,wang2025active}.

\section{Preliminaries}
\textbf{Definition 1: Retrieval Knowledge Base.}
A retrieval knowledge base is defined as two indexed multi-modal repositories,
$B=\{B^{img}, B^{aud}\}$.
The image bank is given by $B^{img}=\{(x_i^{img}, m_i^{img})\}_{i=1}^{N}$, where each image item $x_i^{img}$ is paired with metadata $m_i^{img}$.
Likewise, the audio bank is given by $B^{aud}=\{(x_j^{aud}, m_j^{aud})\}_{j=1}^{M}$, where each ASR transcript $x_j^{aud}$ is paired with metadata $m_j^{aud}$.

\textbf{Definition 2: Tool-Calling Program.}
A tool-calling program is a stepwise list of tool invocations, $\Pi=[u_t]_{t=1}^{T}$, where each invocation $u_t$ is selected from an external tool set $U$ that can interact with an OmniLLM and returns an observation $o_t$.
In OmniRAG-Agent, $U$ includes two types of retrieval tools, namely audio retrieval and image retrieval, $U=\{\textsc{RetrieveIMG}(\cdot), \textsc{RetrieveAUD}(\cdot)\}$.

\textbf{Problem Statement.}
In OmniRAG-Agent, given a natural language question $Q$, an input audio-video stream $X$, and $B$, the goal is to interact with $B$ via $U$ to obtain fine-grained evidence and then generate the final response to $Q$.
\begin{figure*}[htbp]
  \centering
  \includegraphics[width=1.0\textwidth]{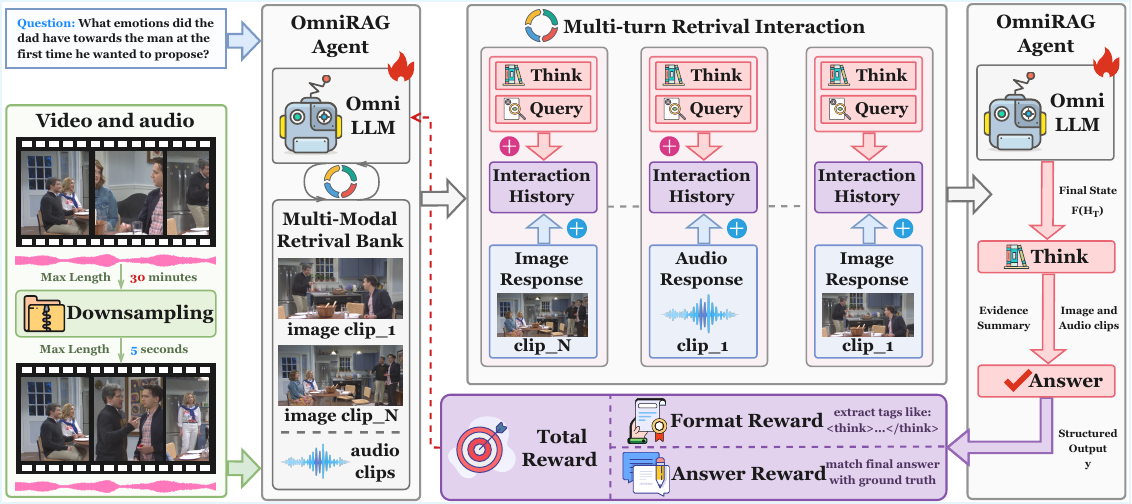}
  \caption{An overview of the OmniRAG-Agent framework. An OmniLLM interacts with a multi-modal retrieval environment to answer long-horizon audio-video questions through multi-turn retrieval. The image and audio retrieval banks are plug-and-play.}
  \label{fig:fig_3}
\end{figure*}
\section{Method: OmniRAG-Agent}
In this section, we introduce OmniRAG-Agent(see Figure~\ref{fig:fig_3}). It includes a multi-modal RAG framework, a multi-turn reasoning process, and a RL-based training method.

\subsection{Multi-Modal RAG}
OmniRAG-Agent build multi-modal bank from the original video, let the model retrieve the most relevant snippets.

\textbf{Multi-Modal Bank Construction.} Since OmniLLMs have a limited maximum input length for videos, we first apply a temporal downsampling operator $\mathcal{D}_t(\cdot)$ to $X$, obtaining a compressed stream $\tilde{X}$ with duration $t$ seconds:
$\tilde{X}=\mathcal{D}_t(X)$. However, such downsampling loses fine-grained information. To recover these clues, we build $B$ from $X$. 

For $B^{img}=\{(x_i^{img}, m_i^{img})\}_{i=1}^{N}$, we sample frame images $\{x_i^{img}\}_{i=1}^{N}$ from $X$ at a fixed interval $\Delta$ seconds. We extract
\begin{equation}
x_i^{img} = \mathrm{Frame}(X,\; (i-1)\Delta), \qquad 
N = \left\lfloor \frac{T_X}{\Delta} \right\rfloor + 1,
\end{equation}
where $T_X$ denotes the overall total time duration of $X$.

For $B^{aud}=\{(x_j^{aud}, m_j^{aud})\}_{j=1}^{M}$, we apply an ASR~\cite{an2025fun} module $\mathrm{ASR}(\cdot)$ to the audio track of $X$ to obtain a sequence of time-stamped segments. Specifically,
\begin{equation}
\{(x_j^{aud}, m_j^{aud})\}_{j=1}^{M} = \mathrm{ASR}(X),
\end{equation}
here $M$ is the number of ASR segments produced from $X$.

\textbf{Retrieving Evidence from the Banks.} During reasoning, the OmniLLM generates natural-language retrieval queries and interacts with $B$ through $U$. At round $t$, the model produces $q_t^{img}$ for image retrieval or $q_t^{aud}$ for audio retrieval.

For image retrieval, given $q_t^{img}$, we encode it with the CLIP~\cite{radford2021learning} text encoder $\phi_{txt}(\cdot)$, where $\phi_{img}(\cdot)$ denotes the corresponding CLIP image encoder, and retrieve the top-$K_{img}$ items from the image bank:
\begin{equation}
s_i^{img}=\operatorname{sim}\!\big(\phi_{txt}(q_t^{img}),\,\phi_{img}(x_i^{img})\big),
\end{equation}
\begin{equation}
\begin{aligned}
E_t^{img}
&=\operatorname{TopK}\!\left(\{(x_i^{img},m_i^{img},s_i^{img})\}_{i=1}^{N},\,K_{img}\right) \\
&=\left\{(x_i^{img},m_i^{img}) \;\middle|\; i\in \mathcal{I}_t^{img}\right\}
\subset B^{img},
\end{aligned}
\end{equation}
where $\operatorname{sim}(\cdot,\cdot)$ denotes a similarity function and $\mathcal{I}_t^{img}$ is the index set of the top-$K_{img}$ scores $\{s_i^{img}\}_{i=1}^{N}$. The retrieved set $E_t^{img}$ provides the timestamps in $m_i^{img}$ that are most likely to contain salient visual evidence relevant to $Q$.

For audio retrieval, given $q_t^{aud}$, we encode it with $\phi_{txt}(\cdot)$ and retrieve the top-$K_{aud}$ items from the audio bank:
\begin{equation}
s_j^{aud}=\operatorname{sim}\!\big(\phi_{txt}(q_t^{aud}),\,\phi_{txt}(x_j^{aud})\big),
\quad
\end{equation}
\begin{equation}
\begin{aligned}
E_t^{aud}
&=\operatorname{TopK}\!\left(\{(x_j^{aud},m_j^{aud},s_j^{aud})\}_{j=1}^{M},\,K_{aud}\right) \\
&=\left\{(x_j^{aud},m_j^{aud}) \;\middle|\; j\in \mathcal{I}_t^{aud}\right\}
\subset B^{aud},
\end{aligned}
\end{equation}
where $\mathcal{I}_t^{aud}$ is the index set of the top-$K_{aud}$ scores $\{s_j^{aud}\}_{j=1}^{M}$. Since $B^{aud}$ is indexed by $\{x_j^{aud}\}_{j=1}^{M}$, this step performs natural-language retrieval over speech content, and the metadata $m_j^{aud}$ provides the corresponding time ranges.
\begin{proposition}
  Multi-modal RAG improves the agent’s ability to solve problems.
\end{proposition}
\begin{proof}
  We provide experimental results in Sections~\ref{sec:main_results}--\ref{sec:generalization} and theoretical justification in Appendix~\ref{sec:prop_1}.
\end{proof}

\subsection{Multi-Turn Reasoning Process}
OmniRAG-Agent supports multi-turn retrieval interactions between the agent and a multimodal bank.

\textbf{Agent Initialization.} Our method adopts a  multi-turn tool-calling framework. The initialization is as follows:
\par\hspace*{1em}(i) Environment $L$. In OmniRAG-Agent, $B$ together with $U$ is treated as the environment. At round $t$, the environment returns a compact evidence observation $o_t$ conditioned on the current history $H_{t-1}$ and the retrieval queries:
\begin{equation}
o_t = E_t \sim B(\cdot \mid H_{t-1}, q_t).
\end{equation}
$o_t$ and $q_t$ are appended to update the interaction history:
\begin{equation}
H_t = H_{t-1}\oplus (q_t, E_t),
\end{equation}
where $\oplus$ denotes appending the new query. At the initial, we set $H_0=[Q,\tilde X,a_{\text{tmpl}}]$, $a_{\text{tmpl}}$ is a fixed prompt template that specifies the tool-calling format and stopping rules.
\par\hspace*{1em}(ii) Agent $\mathcal{M}$. The OmniLLM $\mathcal{M}$ acts as the agent. Given $Q$, $\tilde X$, and $a_{\text{tmpl}}$, the agent first produces the first-round planning trace $z_1$ and the initial retrieval queries:
\begin{equation}
z_1,\;q_1 \sim \pi_\theta(\cdot\mid Q,\tilde X,a_{\text{tmpl}}).
\end{equation}
Then, $q_1$ are sent to the $L$ via the retrieval tools, and the returned observation $o_1=E_1$ is combined to form the first-round history $H_1$ for the next round of multi-turn reasoning.

\textbf{The Agent State Space $\mathcal{H}$.}
The agent state $h_t$ is defined by the multi-turn tool-calling interaction history. At $t$, $\mathcal{M}$ issues $q_t$, receives $o_t$ from $L$, and updates its internal trace $z_t$. We define the state space and its evolution as follows:
\par\hspace*{1em}(i) Initial State ($h_0$). The initial state is constructed from $Q$, $\tilde X$, and $a_{\text{tmpl}}$. Specifically, we set it as:
\begin{equation}
h_0 = [\, Q \oplus \tilde X \oplus a_{\text{tmpl}} \,],
\end{equation}
After the first-round planning and tool calls, the first-round state $h_1$ additionally includes the planning trace $z_1$, the issued queries, and the returned retrieved evidence:
\begin{equation}
h_1 = [\, h_0 \oplus z_1 \oplus (q_1, E_1) \,].
\end{equation}

\par\hspace*{1em}(ii) State Update ($h_t$). From $t\ge 2$, the state update depends on the previous state $h_{t-1}$ and the current round interaction, including $z_t$, $q_t$, and observations, like:
\begin{equation}
h_t = [\, h_{t-1} \oplus z_t \oplus (q_t, E_t) \,].
\end{equation}

\par\hspace*{1em}(iii) State Representation ($F(h_t)$). Since the raw history $h_t$ can grow with the number of rounds, we further maintain a compact state representation $F(h_t)$ that summarizes the complete interaction history. At $t$, $F(h_t)$ is updated by combining the previous representation $F(h_{t-1})$ with the current round’s planning trace and retrieved evidence:
\begin{equation}
F(h_t) = S_t\!\left(F(h_{t-1}),\, z_t,\, E_t\right),
\end{equation}
where $S_t(\cdot)$ denotes a temporal summarization operator. 

\textbf{The Agent Action Space.} At $t$, $\mathcal{M}$ takes an action $a_t$ based on $F(h_{t-1})$. The action includes (i) a short plan for what to find next, (ii) the retrieval queries sent to the tools, and (iii) a stop or continue decision. We define it as:
\begin{equation}
a_t = \big(z_t,\; q_t,\; c_t\big),
\end{equation}
$c_t$ indicates whether to still continue the interaction.

$\mathcal{M}$ follows a stochastic policy $\pi_\theta$ and splits the action probability: generating the queries, and deciding when to stop:
\begin{equation}
\begin{aligned}
\log \pi_\theta(a_t \mid F(h_{t-1})) \;=\;&
\log \pi_\theta\!\left(z_t, q_t \mid F(h_{t-1})\right) \\
&+\; \log \pi_\theta\!\left(c_t \mid F(h_{t-1}), z_t, o_t\right).
\end{aligned}
\end{equation}

\textbf{The Agent Target $(h_T, F(h_T), y)$.} After multi-turn interactions with $L$, $\mathcal{M}$ outputs the final answer to $Q$.
\par\hspace*{1em}(i) Final State. The interaction ends at $T$ when the stopping decision is made or when the maximum number of rounds is reached. The final answer is given by:
\begin{equation}
y = \arg\max_{y\in \mathcal{V}^\ast}\; \pi_\theta\!\left(y \mid Q, a_{\text{tmpl}}, H_T\right),
\end{equation}
where $\mathcal{V}^\ast$ denotes the space of output token sequences and $y$ is the final predicted answer produced by $\mathcal{M}$.

\par\hspace*{1em}(ii) Final Distribution.
The joint distribution of the multi-turn tool-calling trajectory and the final answer is:
\begin{equation}
\begin{aligned}
P_\theta(\tau,y \mid Q,a_{\text{tmpl}},L)
&=
\prod_{t=1}^{T}
\pi_\theta(a_t \mid F(h_{t-1})) \\
&\qquad\;\;
L(o_t \mid H_{t-1}, q_t) \\
&\qquad\;\;
\cdot
\pi_\theta(y \mid Q,a_{\text{tmpl}},H_T),
\end{aligned}
\end{equation}
where $\tau=\{(a_t,o_t)\}_{t=1}^{T}$ is the interaction trajectory.

\begin{proposition}
Multi-turn interaction improves the agent's ability to complete long-horizon tasks.
\end{proposition}

\begin{proof}
We provide experimental results in Sections~\ref{sec:main_results}--\ref{sec:generalization}. Theoretical justification is provided in Appendix~\ref{sec:prop_2}.
\end{proof}

\subsection{RL-Based Optimization}
We optimize the agent's policy using a double constraint reward based end-to-end reinforcement learning.

\textbf{Reward Function.} To enforce valid tool-calling traces and correct answers, we define a constrained reward at the trajectory level. The overall reward $R$ consists of two components: a format reward $R_{\text{fmt}}$ and a performance reward $R_{\text{perf}}$.
\par\hspace*{1em}(i) Format Reward. We enforce $\mathcal{M}$ to follow a predefined structured output format. The format reward is defined as:
\begin{equation}
R_{\text{fmt}} = \min\!\left(1.0,\; 0.5 \sum_{k=1}^{K} \mathbb{I}[t_k \text{ is matched}] \right),
\end{equation}
where $K$ denotes the number of required format tags, $t_k$ is the $k$-th tag, and $\mathbb{I}[\cdot]$ is the indicator function.

\par\hspace*{1em}(ii) Performance Reward. Let $\hat{y}$ denote the predicted final answer parsed from the answer tag, and let $y^\ast$ be the ground-truth answer. We use a binary exact-match reward:
\begin{equation}
R_{\text{perf}}=\mathbb{I}[\hat{y}=y^\ast],
\end{equation}
where $R_{\text{perf}}=1$ if the predicted answer exactly matches the ground truth, and $0$ otherwise.

\par\hspace*{1em}(iii) Gated Composition for Double Constrained Reward. $R$ combines $R_{\text{fmt}}$ and $R_{\text{perf}}$ with a simple gating rule. Specifically, we only credit the performance reward when the format reward reaches a minimum threshold:
\begin{equation}
R=
\begin{cases}
-1.0 + R_{\text{fmt}} + R_{\text{perf}}, & R_{\text{fmt}}\ge 0.5,\\
-1.0 + R_{\text{fmt}}, & \text{otherwise}.
\end{cases}
\end{equation}

\textbf{End-to-End Reinforcement Learning.} We adopt a GRPO-based objective to optimize $\pi_\theta$ with the $R$. Given a batch of $M$ sampled trajectories $\{\tau^{(i)}\}_{i=1}^{M}$, we standardize rewards within the batch to obtain a stable advantage signal. Let $R^{(i)}$ be the reward of trajectory $\tau^{(i)}$ and $\bar{R}$ be the mean reward in the batch. The standardized advantage is:
\begin{equation}
\hat{A}^{(i)}=
\frac{R^{(i)}-\bar{R}}
{\sqrt{\frac{1}{M}\sum_{j=1}^{M}(R^{(j)}-\bar{R})^2+\epsilon}},
\end{equation}
where in our implementation $\hat{A}^{(i)}$ is the standardized advantage and $\epsilon$ is a small constant for numerical stability.

The GRPO-based objective is thus given by:
\begin{equation}
\begin{aligned}
J_{\text{GRPO}}(\theta)
=
\mathbb{E}_{\tau\sim p_\theta(\tau)}\Bigg[
&\frac{1}{M}\sum_{i=1}^{M}\frac{1}{|\tau^{(i)}|}
\sum_{t=1}^{|\tau^{(i)}|}
\min\Big(
r_t^{(i)}\,\hat{A}^{(i)}, \\
&\qquad \mathrm{clip}(r_t^{(i)},1\pm\epsilon)\,\hat{A}^{(i)}
\Big) \\
&\;-\;\beta D_{\text{KL}}(\pi_\theta \,\|\, \pi_{\text{ref}})
\Bigg],
\end{aligned}
\end{equation}
where $w_t^{(i)}$ denotes the $t$-th token in trajectory $\tau^{(i)}$, and for simplicity, we further define the policy ratio is:
\begin{equation}
r_t^{(i)}=
\frac{\pi_\theta(w_t^{(i)}\mid \tau^{(i)}_{<t})}
{\pi_{\theta_{\text{old}}}(w_t^{(i)}\mid \tau^{(i)}_{<t})}.
\end{equation}

Here $\pi_{\theta_{\text{old}}}$ and $\pi_{\text{ref}}$ are the pre-update and reference policies. The $\mathrm{clip}(\cdot)$ operation restricts the policy ratio to $1\pm\epsilon$ to stabilize policy updates. The KL term $D_{\text{KL}}(\pi_\theta \,\|\, \pi_{\text{ref}})$ regularizes the policy toward $\pi_{\text{ref}}$, with $\beta$ controlling its strength.

\begin{proposition}
Reinforcement learning improves the agent's ability to solve tasks.
\end{proposition}

\begin{table*}[ht]
\caption{Comparison of the baselines and the proposed method on OmniVideoBench, \textbf{bold} means the best average performance method.}
\centering
\small
\renewcommand{\arraystretch}{1.1}
\setlength{\tabcolsep}{3pt}
\begin{tabular}{l c ccccccccc}
\toprule
\textbf{Method} & \textbf{Modality} &
\textbf{\makecell{Compare\\Attr}} &
\textbf{AudioBg} &
\textbf{Reasoning} &
\textbf{\makecell{Logic\\Ref}} &
\textbf{\makecell{Ego\\Spatial}} &
\textbf{Perception} &
\textbf{TimeTemp} &
\textbf{\makecell{Text\\Sense}} &
\textbf{Avg} \\
\midrule

\multicolumn{11}{c}{\textbf{\emph{Closed-source Models}}} \\
\midrule
GPT-5.1 & V & 22.22 & 27.27 & 20.83 & 26.09 & 29.17 & 19.67 & 4.00 & 40.62 & 24.51 \\
\midrule
Gemini 2.0-Flash & A+V & 22.22 & 36.36 & 25.00 & 30.43 & 26.76 & 29.51 & 16.00 & 31.25 & 27.34 \\
+ RAG & A+V & 44.44 & 36.36 & 12.50 & 43.48 & 28.17 & 24.59 & 20.00 & 46.88 & 29.69 \\
+ RAG + Agent & A+V & 22.22 & 36.36 & 16.67 & 34.78 & 29.58 & 29.51 & 24.00 & 46.88 & 30.47 \\
\midrule
Gemini 2.5-Flash & A+V & 22.22 & 18.18 & 25.00 & 34.78 & 35.21 & 31.15 & 4.00 & 31.25 & 28.52 \\
+ RAG & A+V & 11.11 & 36.36 & 20.83 & 30.43 & 33.80 & 36.07 & 12.00 & 53.12 & 32.42 \\
+ RAG + Agent & A+V & 44.44 & 27.27 & 29.17 & 30.43 & 33.80 & 34.43 & 28.00 & 53.12 & 35.16 \\
\midrule

\multicolumn{11}{c}{\textbf{\emph{Open-source Models}}} \\
\midrule
Qwen2.5-Omni-3B & A+V & 22.22 & 9.09 & 25.00 & 17.39 & 25.35 & 18.03 & 28.00 & 31.25 & 23.05 \\
+ RAG & A+V & 33.33 & 27.27 & 33.33 & 26.09 & 26.76 & 16.39 & 24.00 & 25.00 & 24.61 \\
+ RAG + Agent & A+V & 22.22 & 26.36 & 37.50 & 26.09 & 25.35 & 24.59 & 24.00 & 28.12 & 26.95 \\
+ RAG + Agent + RL & A+V & 33.33 & 27.27 & 33.33 & 21.74 & 26.76 & 21.31 & 36.00 & 31.25 & 27.34 \\
\midrule
Qwen2.5-Omni-7B & A+V & 22.22 & 18.18 & 25.00 & 26.09 & 23.94 & 27.87 & 16.00 & 31.25 & 25.00 \\
+ RAG & A+V & 22.22 & 27.27 & 12.50 & 39.13 & 22.54 & 24.59 & 40.00 & 40.62 & 27.73 \\
+ RAG + Agent & A+V & 44.44 & 45.45 & 29.17 & 26.09 & 28.17 & 22.95 & 24.00 & 34.52 & 28.52 \\
+ RAG + Agent + RL & A+V & \textbf{33.33} & \textbf{27.27} & \textbf{12.50} & \textbf{26.09} & \textbf{32.39} & \textbf{34.43} & \textbf{20.00} & \textbf{37.50} & \textbf{29.69} \\
\midrule
Qwen3-Omni-30B & A+V & 33.33 & 18.18 & 12.50 & 26.09 & 32.39 & 32.79 & 12.00 & 34.38 & 27.73 \\
+ RAG & A+V & 11.11 & 36.36 & 8.33 & 26.09 & 38.03 & 26.23 & 20.00 & 34.38 & 28.12 \\
+ RAG + Agent & A+V & 33.33 & 36.36 & 16.67 & 21.74 & 38.03 & 27.87 & 28.00 & 25.00 & 29.30 \\
\bottomrule
\end{tabular}
\label{tab:videoomnibench}
\end{table*}

\begin{proof}
We provide empirical evidence in Section~\ref{sec:ablation_study}. Theoretical justification is provided in Appendix~\ref{sec:prop_3}.
\end{proof}

\begin{table*}[ht]
\caption{Comparison of the baselines and the proposed method on WorldSense, \textbf{bold} means the best average performance method.}
\centering
\small
\renewcommand{\arraystretch}{1.1}
\setlength{\tabcolsep}{3pt}
\begin{tabular}{l c ccccccccc}
\toprule
\textbf{Method} & \textbf{Modality} &
\textbf{\makecell{Tech \&\\Science}} &
\textbf{\makecell{Culture \&\\Politics}} &
\textbf{\makecell{Daily\\Life}} &
\textbf{\makecell{Film \&\\TV}} &
\textbf{\makecell{Perfor-\\mance}} &
\textbf{Games} &
\textbf{Sports} &
\textbf{Music} &
\textbf{Avg} \\
\midrule

\multicolumn{11}{c}{\textbf{\emph{Closed-source Models}}} \\
\midrule
Gemini 2.5-Flash & A+V & 47.37 & 23.08 & 27.78 & 40.00 & 27.27 & 29.41 & 25.00 & 43.24 & 33.59 \\
+ RAG & A+V & 55.26 & 38.46 & 29.63 & 36.67 & 31.82 & 41.18 & 21.88 & 45.95 & 37.50 \\
+ RAG + Agent & A+V & 57.89 & 30.77 & 42.59 & 33.33 & 31.82 & 41.18 & 25.00 & 45.95 & 39.84 \\
\midrule

\multicolumn{11}{c}{\textbf{\emph{Open-source Models}}} \\
\midrule
Qwen2.5-Omni-3B & A+V & 34.21 & 42.32 & 27.78 & 30.00 & 13.64 & 29.41 & 15.62 & 40.54 & 29.69 \\
+ RAG & A+V & 34.21 & 26.92 & 25.93 & 30.00 & 31.82 & 41.18 & 31.25 & 35.14 & 31.25 \\
+ RAG + Agent & A+V & 36.84 & 26.92 & 31.48 & 26.67 & 27.27 & 29.41 & 40.62 & 45.95 & 33.98 \\
+ RAG + Agent + RL & A+V & 36.84 & 30.77 & 40.74 & 40.00 & 36.36 & 35.29 & 15.62 & 51.35 & 36.71 \\
\midrule
Qwen2.5-Omni-7B & A+V & 34.21 & 30.77 & 29.63 & 33.33 & 22.73 & 29.41 & 18.75 & 40.54 & 30.47 \\
+ RAG & A+V & 21.05 & 42.31 & 37.04 & 40.00 & 13.64 & 29.41 & 25.00 & 45.95 & 32.81 \\
+ RAG + Agent & A+V & 31.58 & 50.00 & 25.93 & 30.00 & 50.00 & 35.29 & 21.88 & 43.24 & 34.38 \\
+ RAG + Agent + RL & A+V & \textbf{44.74} & \textbf{23.08} & \textbf{40.74} & \textbf{40.00} & \textbf{40.91} & \textbf{41.18} & \textbf{25.00} & \textbf{45.95} & \textbf{38.28} \\
\bottomrule
\end{tabular}
\label{tab:worldsense}
\end{table*}

\begin{table*}[ht]
\caption{Comparison of the baselines and the proposed method on DailyOmni, \textbf{bold} means the best average performance method.}
\centering
\small
\renewcommand{\arraystretch}{1.1}
\setlength{\tabcolsep}{3pt}
\begin{tabular}{l c ccccccc}
\toprule
\textbf{Method} & \textbf{Modality} &
\textbf{\makecell{AV Event\\Alignment}} &
\textbf{Comparative} &
\textbf{\makecell{Context\\Understanding}} &
\textbf{\makecell{Event\\Sequence}} &
\textbf{Inference} &
\textbf{Reasoning} &
\textbf{Avg} \\
\midrule

\multicolumn{9}{c}{\textbf{\emph{Closed-source Models}}} \\
\midrule
Gemini 2.5-Flash & A+V & 34.63 & 34.38 & 50.00 & 37.10 & 56.82 & 51.11 & 44.53 \\
+ RAG & A+V & 38.46 & 39.39 & 45.65 & 41.94 & 54.55 & 55.56 & 46.48 \\
+ RAG + Agent & A+V & 26.92 & 34.38 & 46.65 & 53.23 & 59.09 & 60.00 & 48.83 \\
\midrule

\multicolumn{9}{c}{\textbf{\emph{Open-source Models}}} \\
\midrule
Qwen2.5-Omni-3B & A+V & 39.22 & 21.43 & 41.46 & 25.76 & 33.33 & 29.73 & 32.03 \\
+ RAG & A+V & 27.45 & 28.57 & 60.98 & 22.73 & 48.48 & 37.84 & 35.94 \\
+ RAG + Agent & A+V & 37.25 & 32.14 & 41.46 & 25.76 & 42.42 & 51.35 & 37.11 \\
+ RAG + Agent + RL & A+V & 39.22 & 28.57 & 41.46 & 40.54 & 36.36 & 51.35 & 40.09 \\
\midrule
Qwen2.5-Omni-7B & A+V & 29.41 & 28.57 & 41.46 & 51.52 & 21.21 & 32.43 & 36.33 \\
+ RAG & A+V & 37.25 & 35.71 & 46.34 & 33.33 & 42.42 & 48.65 & 39.84 \\
+ RAG + Agent & A+V & 45.10 & 39.29 & 56.10 & 33.33 & 48.48 & 35.14 & 42.19 \\
+ RAG + Agent + RL & A+V & \textbf{27.45} & \textbf{32.14} & \textbf{51.22} & \textbf{59.09} & \textbf{45.45} & \textbf{45.95} & \textbf{44.92} \\
\bottomrule
\end{tabular}
\label{tab:dailyomni}
\end{table*}

\section{Experiments}
\begin{figure}[!t]
    \centering 
    \includegraphics[width=\columnwidth]{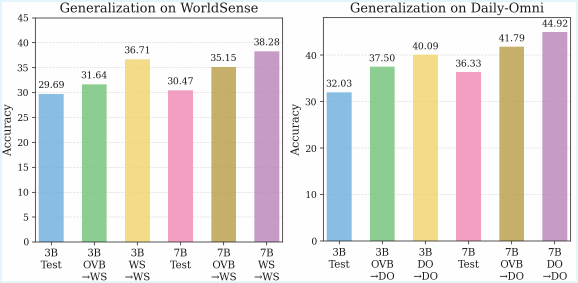}
    \caption{Generalization results across benchmarks. \textbf{OVB} = \textbf{OmniVideoBench}, \textbf{WS} = \textbf{WorldSense}, \textbf{DO} = \textbf{Daily-Omni}. \textbf{“→”} means train on the dataset before the arrow and test on the dataset after the arrow.}
    \label{fig:generalize}
\end{figure}
In this section, we present the experimental setup and results of OmniRAG-Agent. We aim to answer the following research questions (RQs): \textbf{RQ1:} Does OmniRAG-Agent consistently outperform existing OmniLLMs under low-resource long audio-video settings? \textbf{RQ2:} How well does OmniRAG-Agent generalize across different long-horizon omnimodal QA benchmarks? \textbf{RQ3:} Is OmniRAG-Agent transferable across different backbone models? \textbf{RQ4:} How effective are the key components in improving tool use and final answer quality? \textbf{RQ5:} How does the budget affect the overall performance of OmniRAG-Agent?

\begin{figure*}[htbp]
  \centering
  \includegraphics[width=1.0\textwidth]{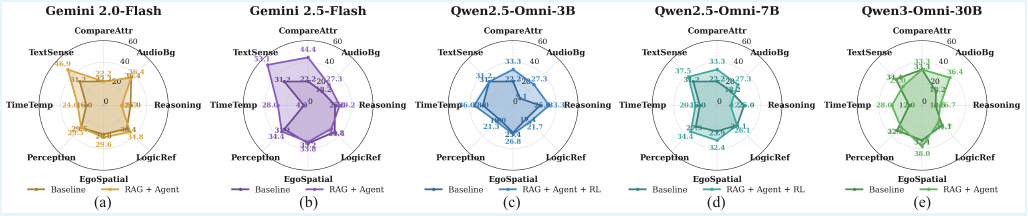}
  \caption{Transferability across different OmniLLM backbones on OmniVideoBench. OmniRAG-Agent is applied to five representative OmniLLMs, including two closed-source models and three open-source models.}
  \label{fig:different_models}
\end{figure*}

\begin{figure*}[htbp]
  \centering
  \includegraphics[width=1.0\textwidth]{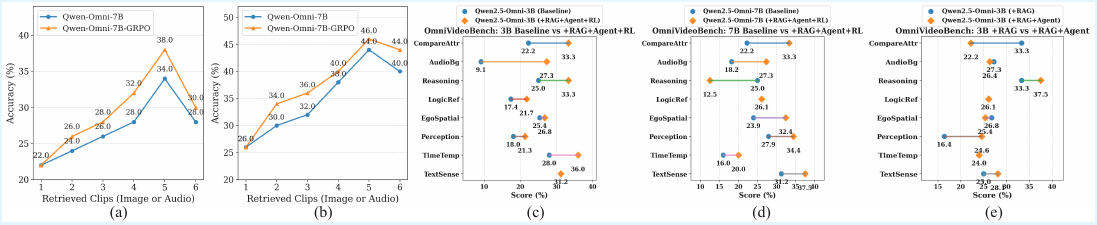}
  \caption{Budget analysis on 50 samples from OmniVideoBench. We vary the retrieval budget and compare with and without RL.}
  \label{fig:budget}
\end{figure*}

\subsection{Experimental Setup}
\textbf{Datasets.} We evaluate OmniRAG-Agent on three long-horizon omnimodal question answering benchmarks: \textbf{OmniVideoBench}~\cite{li2025omnivideobench}, \textbf{WorldSense}~\cite{hong2025worldsense}, and \textbf{Daily-Omni}~\cite{zhou2025daily}. These datasets require reasoning over long audio-video streams~\cite{cao2025videominer} where key evidence is scattered across time. OmniVideoBench provides fine-grained ability-based subsets, allowing us to analyze where retrieval and multi-turn tool use bring the most benefit. WorldSense and DailyOmni further test the robustness of the method~\cite{cao2025xgc,zhang2025deep,li2024baichuan} on diverse real-world scenarios with multimodal clues from both what is seen and what is heard. More details are in Appendix~\ref{sec:datasets}.

\textbf{Baselines.} We compare OmniRAG-Agent against representative closed-source and open-source OmniLLMs under low-resource settings. For closed-source models, we report results on \textbf{GPT-5.1}~\cite{ma2026safety} and \textbf{Gemini}~\cite{comanici2025gemini} variants. For open-source models, we evaluate \textbf{Qwen2.5-Omni}~\cite{xu2025qwen2,xu2025qwen3omnitechnicalreport} backbones. To isolate the effect of each component, we additionally include incremental variants built on the same backbone: \emph{(i)} \textbf{Base} (direct long-horizon QA without external retrieval), \emph{(ii)} \textbf{Base}+\textbf{RAG} (single-step retrieval-augmented QA), \emph{(iii)} \textbf{Base}+\textbf{RAG}+\textbf{Agent} (multi-turn planning and tool calling), and \emph{(iv)} \textbf{Base}+\textbf{RAG}+\textbf{Agent}+\textbf{RL} (RL-based~\cite{dong2025agentic} optimization). More details are in Appendix~\ref{sec:baselines}.

\textbf{Evaluation Metrics.} We follow the official evaluation protocols of each benchmark. For these datasets, we report accuracy on each ability subset and their average score, which measures overall performance across heterogeneous reasoning skills.

\textbf{Implementation Details.} We implement OmniRAG-Agent with an image--audio retrieval environment consisting of an image bank and an ASR-indexed audio bank. The image bank is constructed by sampling frames from the original video at a fixed interval, and the audio bank is built from time-stamped ASR segments. At inference time, the agent iteratively plans, issues retrieval queries, and integrates retrieved evidence. For RL training, we optimize the agent policy using GRPO with a gated reward that jointly enforces valid tool-calling traces and improves final answer quality. More details are in Appendix~\ref{sec:tools} and Appendix~\ref{sec:hyperparameters}.

\subsection{Main Results (RQ1)}
\label{sec:main_results}
As shown in Tables~\ref{tab:videoomnibench}--\ref{tab:dailyomni}, OmniRAG-Agent consistently improves the corresponding baselines across all three benchmarks under the same low-resource setting, with gains spanning multiple ability dimensions, including Reasoning~\cite{feng2025video}, Perception, etc. We take Table~\ref{tab:videoomnibench} as an example. For open-source backbones, incorporating external evidence retrieval, multi-turn agentic reasoning and RL consistently improves long-horizon audio--visual question answering over direct inference.

\subsection{Generalization Results (RQ2)}
\label{sec:generalization}
As shown in Figure~\ref{fig:generalize}, OmniRAG-Agent generalizes well when we move across benchmarks. In particular, models that are trained on OmniVideoBench can be directly transferred to WorldSense or Daily-Omni and still get a clear boost over zero-shot testing, which suggests the learned retrieval-and-reasoning~\cite{guan2025deeprag} behavior is not tied to a single dataset style. The gains are generally more noticeable for the smaller backbone, while the larger backbone also benefits but shows a mild diminishing-return pattern since it already starts from a stronger baseline.

\subsection{Work with Other OmniLLMs (RQ3)}
We also test how well OmniRAG-Agent works with different OmniLLM backbones by plugging it into five representative models on OmniVideoBench. As shown in Figure~\ref{fig:different_models}, adding OmniRAG-Agent consistently pushes the performance curves outward across most ability dimensions. Across models, the largest improvements tend to appear on abilities that rely more on fine-grained grounding and cross-modal evidence, such as reasoning, temporal understanding, and text-related sense~\cite{wang2023internvid}, indicating that OmniRAG-Agent enhances how different OmniLLMs use evidence rather than changing their core strengths.

\subsection{Ablation study (RQ4)}
\label{sec:ablation_study}
We run ablations to evaluate the contribution of each module. The overall trend is quite consistent across the three benchmarks in Tables~\ref{tab:videoomnibench}--\ref{tab:dailyomni}. With the same backbone, adding RAG is usually where we see the first obvious improvement, likely because the model can pull in short, relevant snippets. Adding the agent loop on top usually gives an extra lift, especially for skills that rely on multi-turn reasoning~\cite{xue2025simpletir} and retrieval. For the open-source models, GRPO often adds a smaller but still worthwhile boost.

\subsection{Budget Analysis (RQ5)}
We study RQ5 in Figure~\ref{fig:budget} by evaluating how the retrieval budget~\cite{gan2025retrieval} (number of retrieved clips) and whether we enable multi-turn retrieval and RL training affect OmniRAG-Agent. From (a)(b), we find that using a moderate number of retrieved audio-video clips yields the best trade-off, reaching strong accuracy while avoiding an overly long context. From (c)(d)(e), both the multi-turn retrieval mechanism and RL consistently push scores upward across ability dimensions, suggesting that increasing inference and training budget to a extent improves performance.

\section{Conclusion}
In this work, we propose OmniRAG-Agent for low-resource long-horizon audio-video QA by enabling an OmniLLM to interact with multi-modal banks through retrieval and multi-turn tool use. We find three key takeaways from our framework: (i) fine-grained evidence can be retrieved under tight budgets, (ii) multi-step agentic planning helps gather and verify scattered clues across long inputs, and (iii) reinforcement learning can further improve evidence selection and final answering in an end-to-end manner.
% In the unusual situation where you want a paper to appear in the
% references without citing it in the main text, use \nocite
\section*{Impact Statement}
This work introduces OmniRAG-Agent, an agentic framework for low-resource long-horizon audio-video question answering that combines multimodal retrieval, multi-turn tool use, and reinforcement learning to improve evidence grounding and reasoning accuracy. While the method shows consistent gains, it still faces limitations such as reliance on retrieval quality, possible multi-turn error accumulation, and added computation from iterative interaction. Future work includes developing more robust stopping and verification strategies, extending to broader multimodal settings, and improving transfer to real-world domains.

\nocite{langley00}

\bibliography{example_paper}
\bibliographystyle{icml2026}

%%%%%%%%%%%%%%%%%%%%%%%%%%%%%%%%%%%%%%%%%%%%%%%%%%%%%%%%%%%%%%%%%%%%%%%%%%%%%%%
%%%%%%%%%%%%%%%%%%%%%%%%%%%%%%%%%%%%%%%%%%%%%%%%%%%%%%%%%%%%%%%%%%%%%%%%%%%%%%%
% APPENDIX
%%%%%%%%%%%%%%%%%%%%%%%%%%%%%%%%%%%%%%%%%%%%%%%%%%%%%%%%%%%%%%%%%%%%%%%%%%%%%%%
%%%%%%%%%%%%%%%%%%%%%%%%%%%%%%%%%%%%%%%%%%%%%%%%%%%%%%%%%%%%%%%%%%%%%%%%%%%%%%%
\newpage
\appendix
\onecolumn
\section{Prompts Used in OmniRAG-Agent}
\subsection{Initial Prompt}

Figure~\ref{fig:init_prompt} illustrates the initial prompt used in the OmniRAG-Agent interaction process. This prompt specifies a strict agentic protocol for budgeted long-horizon audio-video QA, where the model iteratively alternates between internal reasoning and a single tool/answer decision at each turn. It enforces a structured XML-like~\cite{wei2022chain} output format with two tags per step (\texttt{\textless think\textgreater} followed by exactly one of \texttt{\textless search\_image\textgreater}, \texttt{\textless search\_audio\textgreater}, or \texttt{\textless answer\textgreater}), and restricts tool queries to natural-language retrieval requests over video frames or audio clips only.

\begin{figure*}[htbp]
  \centering
  \includegraphics[width=1.0\textwidth]{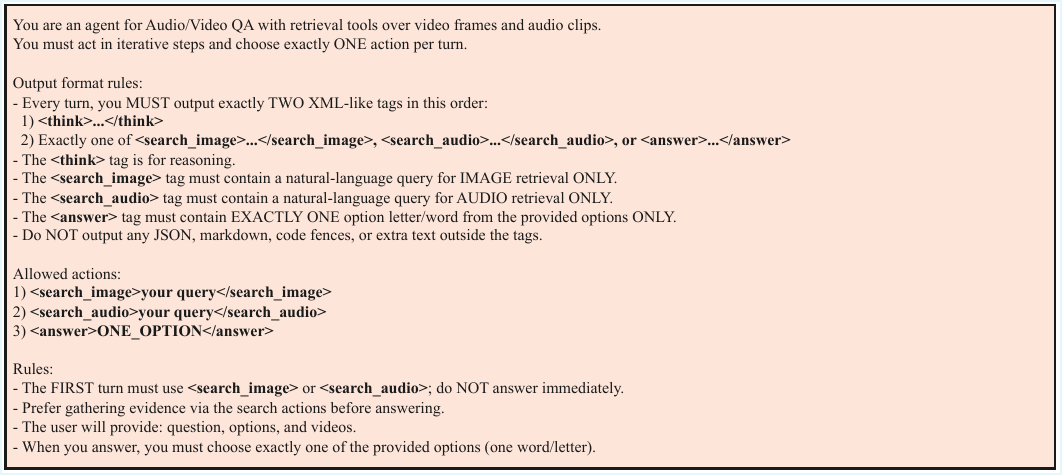}
  \caption{Initial prompt of agent process.}
  \label{fig:init_prompt}
\end{figure*}

\subsection{Multi-turn Interaction Process}
Figure~\ref{fig:multi_turn_process} shows a complete multi-turn interaction example of OmniRAG-Agent for a given question. The figure illustrates how the agent iteratively performs retrieval-augmented reasoning using a Thought--Action--Observation style prompt~\cite{yao2023tree,besta2024graph,qiao2024autoact}: at each step, the model reasons about the current context and evidence summary (Thought), selects a retrieval tool and issues a natural-language query to either the image bank or the audio bank (Action), and then incorporates the returned time-localized snippets and metadata as new evidence (Observation).

\begin{figure*}[htbp]
  \centering
  \includegraphics[width=1.0\textwidth]{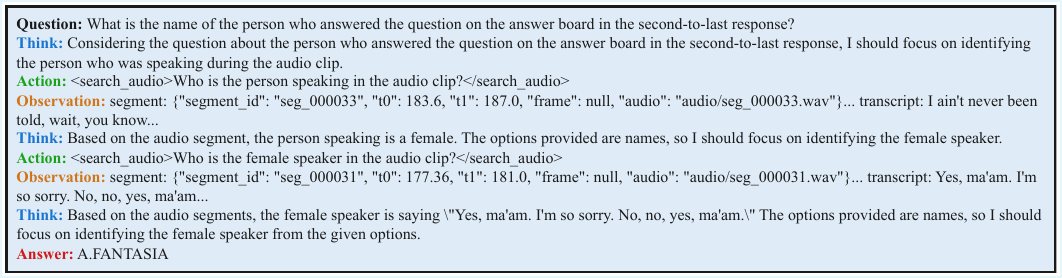}
  \caption{An example of a complete agent process.}
  \label{fig:multi_turn_process}
\end{figure*}

\section{Proof}
\subsection{Proof of Proposition 4.1}
\label{sec:prop_1}
\begin{proposition}
Multi-modal RAG improves the agent's ability to solve problems.
\end{proposition}

\begin{proof}
Consider a long-horizon audio-video QA instance with question $Q$, compressed input $\tilde{X}$ (budget-limited), and the multimodal retrieval banks
$B=\{B^{img},B^{aud}\}$.
Let $Y^\star$ denote the ground-truth answer and $\hat{Y}$ the agent’s predicted answer.
We assume the question is evidence-sparse, i.e., there exists a small set of necessary evidence snippets $E^\star$ (potentially across modalities) such that accessing at least one element of $E^\star$ significantly increases the probability of answering correctly.

Let $G$ be the event that the RAG process retrieves at least one necessary evidence snippet into the context during inference:
\begin{equation}
G := \left\{\exists\, t \le T \text{ such that } E_t \cap E^\star \neq \emptyset \right\},
\qquad
P(G)=r>0.
\end{equation}

Define two conditional success probabilities:
\begin{equation}
\alpha_1 := P(\hat{Y}=Y^\star \mid G),
\qquad
\alpha_0 := P(\hat{Y}=Y^\star \mid \neg G),
\end{equation}
and assume $\alpha_1>\alpha_0$, i.e., seeing necessary evidence increases success probability.
Then, by the law of total probability,
\begin{equation}
\mathbb{E}\!\left[\mathrm{Acc}_{\mathrm{RAG}}\right]
= P(G)\,P(\hat{Y}=Y^\star \mid G) + P(\neg G)\,P(\hat{Y}=Y^\star \mid \neg G)
\ge r\alpha_1 + (1-r)\alpha_0.
\end{equation}

Under a tight budget where the baseline agent only observes $\tilde{X}$ without retrieval, the baseline is effectively constrained to the $\neg G$ regime (it has no mechanism to surface $E^\star$ from the raw long stream). Hence,
\begin{equation}
\mathbb{E}\!\left[\mathrm{Acc}_{\mathrm{Base}}\right] \le \alpha_0
\;\Rightarrow\;
\mathbb{E}\!\left[\mathrm{Acc}_{\mathrm{RAG}}\right]
-
\mathbb{E}\!\left[\mathrm{Acc}_{\mathrm{Base}}\right]
\ge r(\alpha_1-\alpha_0) > 0.
\end{equation}
This yields a positive lower bound on the expected accuracy gain induced by retrieval.

\paragraph{Multi-modal advantage.}
Suppose the necessary evidence decomposes into visual and audio components
$E^\star = E^\star_{img} \cup E^\star_{aud}$.
Let $r_{img}$ and $r_{aud}$ denote the probabilities that RAG retrieves at least one necessary snippet from the image bank and audio bank, respectively.
Assuming the two retrieval channels are not adversarially coupled, the probability of retrieving at least one necessary snippet from either modality satisfies
\begin{equation}
P(G_{mm})
=
1-(1-r_{img})(1-r_{aud})
\ge \max(r_{img},r_{aud}),
\end{equation}
so the above lower bound becomes larger when using the union $B^{img}\cup B^{aud}$, explaining why multi-modal RAG is particularly beneficial for cross-modal and temporally sparse evidence.

\textbf{Summary.}
Multi-modal RAG increases the probability of bringing necessary evidence into the agent’s context ($r>0$), and since correctness conditioned on having such evidence is higher ($\alpha_1>\alpha_0$), the expected task accuracy admits a strictly positive improvement lower bound $r(\alpha_1-\alpha_0)$.
\end{proof}

\subsection{Proof of Proposition 4.2}
\label{sec:prop_2}
\begin{proposition}
Multi-turn interaction improves the agent's ability to complete long-horizon tasks.
\end{proposition}

\begin{proof}
Let the answer set be $\mathcal{A}=\{1,\ldots,M\}$ and assume the true answer is a hidden variable $Y\in\mathcal{A}$.
At interaction round $t$, the agent issues a retrieval action (query) $q_t$ based on the previous interaction history
\begin{equation}
H_{t-1}=\{(q_1,o_1),\ldots,(q_{t-1},o_{t-1})\},
\end{equation}
where $o_s$ denotes the retrieved observation (evidence) returned by the retrieval environment at round $s$.
The retrieval environment produces an observation conditioned on $Y$:
\begin{equation}
o_t \sim K_{q_t}(\cdot \mid Y),
\end{equation}
where $K_{q_t}(\cdot \mid Y)$ is the class-conditional observation law indexed by the query/action $q_t$.
The history is then updated as
\begin{equation}
H_t = H_{t-1}\oplus (q_t,o_t),
\end{equation}
where $\oplus$ denotes appending the ordered pair $(q_t,o_t)$ to the history.

Define the posterior vector and Bayes accuracy function as
\begin{equation}
\pi_t(y) \triangleq P(Y=y \mid H_t),
\qquad
A(H_t) \triangleq \max_{y\in\mathcal{A}} \pi_t(y),
\end{equation}
where $A(H_t)$ is the Bayes accuracy under $0$-$1$ loss.
Introduce the Bayes risk potential function to measure uncertainty:
\begin{equation}
V(H_t) \triangleq 1-A(H_t)=1-\max_{y\in\mathcal{A}}\pi_t(y),
\end{equation}
where smaller values indicate lower uncertainty.

\textbf{(i) Posterior martingale and expectation contraction.}
Let $\mathcal{F}_{t-1}=\sigma(H_{t-1})$ be the natural filtration generated by the interaction history.
By Bayes' rule, the posterior satisfies the martingale property
\begin{equation}
\mathbb{E}[\pi_t(y)\mid \mathcal{F}_{t-1}] = \pi_{t-1}(y), \qquad \forall y\in\mathcal{A}.
\end{equation}
Define a concave potential function over the probability simplex:
\begin{equation}
\phi(p) \triangleq 1-\max_{y} p_y, \qquad p\in\Delta^{M-1}.
\end{equation}
Since $\phi(\cdot)$ is concave and $\pi_t$ is a martingale, Jensen's inequality yields
\begin{equation}
\mathbb{E}[V(H_t)\mid \mathcal{F}_{t-1}] \le V(H_{t-1}),
\end{equation}
where the inequality is strict whenever the observation kernel $K_{q_t}(\cdot \mid Y)$ is information-bearing (i.e., not independent of $Y$).

\textbf{(ii) Monotone improvement over multiple turns.}
Taking unconditional expectation and iterating the above relation gives
\begin{equation}
\mathbb{E}[V(H_t)] \le \mathbb{E}[V(H_{t-1})] \le \cdots \le \mathbb{E}[V(H_0)].
\end{equation}
Define the one-step expected reduction of Bayes risk by
\begin{equation}
\Delta_t \triangleq \mathbb{E}\!\left[V(H_{t-1})-\mathbb{E}[V(H_t)\mid \mathcal{F}_{t-1}]\right] \ge 0.
\end{equation}
Then the expected Bayes risk after $t$ rounds satisfies
\begin{equation}
\mathbb{E}[V(H_t)] = \mathbb{E}[V(H_0)] - \sum_{s=1}^{t}\Delta_s,
\end{equation}
and substituting $A(H_t)=1-V(H_t)$ yields
\begin{equation}
\mathbb{E}[A(H_t)] = 1-\mathbb{E}[V(H_0)] + \sum_{s=1}^{t}\Delta_s.
\end{equation}
Therefore, under informative observations, multi-turn adaptive retrieval ensures that the expected Bayes accuracy is monotone non-decreasing with the number of interaction rounds.

\textbf{(iii) Strict advantage for scattered evidence.}
When the necessary evidence is distributed across multiple temporal segments/modalities, a single retrieval round may fail to surface sufficient information for disambiguation, resulting in $\Delta_t \approx 0$.
In contrast, multi-turn adaptive retrieval can condition future queries on the accumulated history $H_{t-1}$, producing additional informative observations and yielding $\Delta_t>0$ for some rounds.
Hence, multi-turn interaction strictly reduces the Bayes risk potential and achieves strictly higher expected accuracy than one-shot retrieval in such scattered-evidence settings.
\end{proof}

\subsection{Proof of Proposition 4.3}
\label{sec:prop_3}
\begin{proposition}
Reinforcement learning improves the agent's ability to solve tasks.
\end{proposition}

\begin{proof}
Let the input question be $Q$, and let the ground-truth answer be a random variable $Y^\star \in \mathcal{A}$, where $\mathcal{A}$ is the answer set.
Let the agent policy be $\pi_\theta$, which generates a multi-turn tool-calling trajectory
$\tau=\{(a_t,o_t)\}_{t=1}^{T}$, where $a_t$ contains the plan/query/stop decision and $o_t$ is the retrieved observation returned by the environment.
The final answer $\hat{Y}$ is produced conditioned on the terminal history $H_T$.
The joint distribution of the trajectory and final answer can be written as
\begin{equation}
P_\theta(\tau,\hat{Y}\mid Q)
=
p_\theta(\tau\mid Q)\,P(\hat{Y}\mid H_T,Q),
\end{equation}
where $p_\theta(\tau\mid Q)$ is the trajectory distribution induced by $\pi_\theta$ through multi-turn interaction.

We optimize $\pi_\theta$ using GRPO with a gated constrained reward and KL regularization.
Let the trajectory-level reward be
\begin{equation}
R(\tau,\hat{Y})
=
R_{\mathrm{fmt}}(\tau) + \mathbb{I}\!\left[R_{\mathrm{fmt}}(\tau)\ge \delta\right] R_{\mathrm{perf}}(\hat{Y},Y^\star),
\end{equation}
where $R_{\mathrm{fmt}}$ measures format validity, $R_{\mathrm{perf}}=\mathbb{I}[\hat{Y}=Y^\star]$ is the exact-match task reward, and $\delta$ is a fixed threshold.
Note that
\begin{equation}
\mathbb{E}[R_{\mathrm{perf}}] = \mathbb{P}(\hat{Y}=Y^\star),
\end{equation}
so maximizing the expected reward increases the success probability, while the gating term ensures that correctness is credited only for valid, parseable tool-use trajectories.

The KL-regularized GRPO objective can be expressed abstractly as
\begin{equation}
J(\pi_\theta)
=
\mathbb{E}_{\tau\sim p_\theta(\tau\mid Q)}\!\left[R(\tau,\hat{Y})\right]
-
\beta D_{\mathrm{KL}}(\pi_\theta \,\|\, \pi_{\mathrm{ref}}),
\end{equation}
where $\pi_{\mathrm{ref}}$ is a reference policy and $\beta>0$ controls the regularization strength.
The first term encourages policies that generate valid tool-calling trajectories and correct answers; the second term constrains the update to a trust region, preventing unstable distribution shifts.

To justify policy improvement, we appeal to the standard performance difference / trust-region analysis used in TRPO/PPO.
For a KL-constrained update, there exists a constant $C>0$ such that the updated policy satisfies the lower bound
\begin{equation}
J(\pi_{\mathrm{new}})
\;\ge\;
J(\pi_{\mathrm{old}})
+
\mathbb{E}_{d_{\pi_{\mathrm{new}}}}\!\left[A_{\pi_{\mathrm{old}}}\right]
-
C\cdot D_{\mathrm{KL}}(\pi_{\mathrm{new}}\| \pi_{\mathrm{old}}),
\end{equation}
where $A_{\pi_{\mathrm{old}}}$ is the advantage function under the old policy and $d_{\pi_{\mathrm{new}}}$ is the state visitation distribution of the new policy.
GRPO implements this trust-region principle via clipping of policy ratios together with the explicit KL penalty, so the update increases the expected return while keeping the policy shift bounded.

Finally, since the task reward is an exact-match indicator and is gated by the format validity,
an increase in $\mathbb{E}[R(\tau,\hat{Y})]$ implies simultaneous improvement in both (i) the probability of producing valid tool-calling trajectories and (ii) the probability of producing the correct final answer among valid trajectories.
This effect is particularly important for long-horizon settings, where the policy must learn multi-step decisions (querying, stopping, and evidence integration) and thus benefits more from stable policy improvement under KL-regularized GRPO.
Therefore, GRPO under constrained rewards improves the agent's long-horizon video-audio QA capability in expectation.
\end{proof}

\section{Algorithm Details}
\paragraph{Overview.}
OmniRAG-Agent is a budgeted long-horizon audio-video QA framework that equips an OmniLLM $\mathcal{M}$ with (i) an external multi-modal retrieval environment and (ii) an agentic multi-turn tool-calling loop, further optimized by end-to-end reinforcement learning.
Given a question $Q$ and a long audio-video stream $X$, OmniRAG-Agent first constructs a compressed stream $\tilde{X}$ under an input budget and builds a plug-and-play multimodal bank $B=\{B^{img},B^{aud}\}$ from the original $X$.
During reasoning, $\mathcal{M}$ iteratively generates a short plan, issues natural-language retrieval queries to image/audio tools, and integrates returned evidence snippets into its evolving history to answer $Q$.
The overall procedure contains three connected stages:
(A) initialization, where the agent is instantiated with the prompt template and the budgeted input $\tilde{X}$, and the retrieval environment is prepared from $X$;
(B) multi-turn interaction, where $\mathcal{M}$ alternates between planning, tool calling (\textsc{RetrieveIMG}/\textsc{RetrieveAUD}), evidence aggregation, and a stop/continue decision until termination;
(C) end-to-end reinforcement learning, where GRPO updates the agent policy using a gated reward that jointly evaluates tool-call format validity and final answer correctness, improving both tool use and long-horizon QA performance over time.

\algrenewcommand\algorithmicrequire{\textbf{Input:}}
\algrenewcommand\algorithmicensure{\textbf{Output:}}
\algnewcommand{\LineComment}[1]{\State \(\triangleright\)~#1}
\begin{algorithm*}[!t]
\caption{OmniRAG-Agent: Multi-turn Omnimodal Retrieval Interaction with End-to-End GRPO Optimization}
\label{alg:omnirag_agent}
\begin{algorithmic}[1]
\Require Question $Q$; long audio-video stream $X$; OmniLLM agent $\mathcal{M}$ with policy $\pi_\theta$;
tool set $U=\{\textsc{RetrieveIMG}(\cdot),\textsc{RetrieveAUD}(\cdot)\}$;
prompt template $a_{\text{tmpl}}$; max turns $T$; termination rule $\mathsf{Stop}(\cdot)$;
bank hyperparameters $(\Delta, K_{img}, K_{aud})$; reward threshold $\delta$;
GRPO hyperparameters (rollouts $N$, clip $\epsilon$, KL coef $\beta$); reference policy $\pi_{\text{ref}}$.
\Ensure Final answer $y$ (in inference); updated parameters $\theta$ (in training).

\LineComment{Stage A: Bank Construction and Agent Initialization}
\State $\tilde{X} \gets \mathcal{D}_t(X)$ \Comment{budgeted temporal downsampling}
\State $B^{img} \gets \{(x_i^{img}, m_i^{img})\}_{i=1}^{N} \;\leftarrow\; \textsc{SampleFrames}(X,\Delta)$
\State $B^{aud} \gets \{(x_j^{aud}, m_j^{aud})\}_{j=1}^{M} \;\leftarrow\; \textsc{ASR}(X)$
\State $B \gets \{B^{img}, B^{aud}\}$
\State $H_0 \gets [\, Q \oplus \tilde{X} \oplus a_{\text{tmpl}} \,]$; \hspace{0.6em} $F(h_0)\gets \textsc{InitSummary}(H_0)$

\LineComment{Stage B: Multi-turn Tool-Calling Interaction}
\For{$t \gets 1$ \textbf{to} $T$}
  \State $(z_t, q_t, c_t) \sim \pi_\theta(\cdot \mid F(h_{t-1}))$ \Comment{plan, query, stop/continue}
  \If{$\mathsf{Stop}(c_t, z_t, F(h_{t-1}))$}
    \State \textbf{break}
  \EndIf
  \If{$q_t$ targets image evidence}
    \State $E_t \gets \textsc{RetrieveIMG}(q_t, B^{img}, K_{img})$
  \Else
    \State $E_t \gets \textsc{RetrieveAUD}(q_t, B^{aud}, K_{aud})$
  \EndIf
  \State $H_t \gets H_{t-1} \oplus (z_t, q_t, E_t)$
  \State $F(h_t) \gets S_t(F(h_{t-1}), z_t, E_t)$ \Comment{compact history/evidence summary}
\EndFor
\State $y \gets \mathcal{M}(\cdot \mid Q, a_{\text{tmpl}}, H_t)$ \Comment{final answer conditioned on terminal history}

\LineComment{Stage C: End-to-End Reinforcement Learning (Training Only)}
\If{$\textsc{Training} = \textsc{True}$}
  \State Sample $N$ trajectories $\{\tau_i\}_{i=1}^{N}$ by repeating Stages A--B
  \For{$i \gets 1$ \textbf{to} $N$}
    \State $R_{\text{fmt}}^{(i)} \gets \textsc{FormatScore}(\tau_i)$ \Comment{tag match / parse validity}
    \State $R_{\text{perf}}^{(i)} \gets \textsc{AnswerScore}(\tau_i)$ \Comment{exact match to $y^\star$}
    \State $R^{(i)} \gets -1 + R_{\text{fmt}}^{(i)} + \mathbb{I}[R_{\text{fmt}}^{(i)} \ge \delta]\cdot R_{\text{perf}}^{(i)}$
  \EndFor
  \State Compute standardized advantages $\{\hat A^{(i)}\}$ from $\{R^{(i)}\}_{i=1}^{N}$
  \State Update $\theta$ with GRPO using ratio clipping $(1\pm \epsilon)$ and KL penalty $\beta D_{\mathrm{KL}}(\pi_\theta\Vert \pi_{\text{ref}})$
\EndIf

\end{algorithmic}
\end{algorithm*}

\paragraph{Training and Inference Flow.}
During training, OmniRAG-Agent runs all three stages (A$\rightarrow$B$\rightarrow$C): it constructs the multimodal banks, executes multi-turn retrieval-augmented interaction to produce a complete tool-calling trajectory, and then updates the policy using GRPO with standardized advantages and KL regularization.
During testing, it runs only (A$\rightarrow$B): bank construction and multi-turn interaction are performed with frozen parameters, after which the final answer is generated directly when the agent emits the stop decision or reaches the maximum number of turns.

\paragraph{Complexity Analysis.}
Let $|X|$ be the raw video duration (or total frames/audio length), and let the image sampling interval be $\Delta$ seconds.
Bank construction costs $\mathcal{O}(|X|/\Delta)$ for frame sampling and $\mathcal{O}(|X|)$ for ASR segmentation (implementation-dependent), and is done once per input.
Multi-turn interaction runs up to $T$ rounds, each containing one OmniLLM generation step and at most one retrieval call, yielding time complexity $\mathcal{O}(T)$ for agent steps plus retrieval scoring over the banks (typically $\mathcal{O}(N)$ or sublinear with ANN indexing, where $N$ is bank size).
Memory grows with the interaction history and retrieved snippets, and can be controlled via evidence windowing or the compact summarization state $F(h_t)$.
Reinforcement learning samples $N$ trajectories per update, each with up to $T$ steps, giving training complexity $\mathcal{O}(NT)$ (plus reward/advantage computation and KL terms).
Overall, training is $\mathcal{O}(NT+T)$ (plus one-time bank construction per sample), and inference is $\mathcal{O}(T)$ (plus retrieval cost).

\section{Dataset Details}
\label{sec:datasets}
\textbf{OmniVideoBench dataset}~\cite{li2025omnivideobench} is a large-scale benchmark specifically designed to evaluate long-horizon audio--visual reasoning capabilities of omni-modal large language models. It consists of high-quality question--answer pairs constructed from real-world videos with durations ranging from several seconds to tens of minutes, requiring models to jointly reason over visual content and complex audio signals such as speech, sound, and music. In our experiments, we use 504 examples for training and 256 examples for testing. The dataset covers a wide spectrum of reasoning types, including temporal understanding, spatial localization, causal inference, counting, and cross-modal alignment, making it particularly challenging under low-resource settings. Moreover, we follow the official question\_type taxonomy and further merge fine-grained OmniVideoBench categories into eight ability groups (e.g., \texttt{Reasoning} $\leftarrow$ \{\textit{causal reasoning}, \textit{hypothetical reasoning}\}, \texttt{Perception} $\leftarrow$ \{\textit{fine-grained perception}, \textit{counting}\}) using a fixed rule-based mapping (cf. \texttt{AGGREGATION\_RULES} in our evaluation script).

\textbf{WorldSense dataset}~\cite{hong2025worldsense} is a comprehensive benchmark designed to evaluate real-world omni-modal understanding, with a strong emphasis on tightly coupled audio--visual reasoning. It is constructed from diverse, synchronized audio--visual videos spanning multiple real-world domains, where neither visual nor audio information alone is sufficient to answer the questions correctly. The dataset includes a wide range of tasks that require models to integrate speech, environmental sounds, music, and visual cues to perform perception, understanding, and high-level reasoning. In our experiments, we use 504 examples for training and 256 examples for testing. Due to its focus on cross-modal dependency, fine-grained temporal alignment, and complex real-world scenarios, WorldSense poses significant challenges for omni-modal models under low-resource settings, making it well-suited for evaluating multi-turn retrieval and agentic reasoning frameworks.

\textbf{Daily-Omni dataset}~\cite{zhou2025daily} is an audio--visual question answering benchmark designed to evaluate multimodal models on temporally aligned cross-modal reasoning in real-world daily life scenarios. It comprises videos rich in both visual dynamics and diverse audio signals, including speech, music, and environmental sounds, requiring models to jointly process and align information across modalities over time. In our experiments, we use 504 examples for training and 256 examples for testing. The questions span multiple reasoning categories such as audio--visual event alignment, event sequencing, contextual understanding, inference, and comparative reasoning, explicitly emphasizing the necessity of precise temporal correspondence between audio and visual events. This dataset is particularly challenging for models that lack fine-grained temporal awareness or rely on unimodal cues, making it well suited for evaluating agent-based, multi-turn retrieval and reasoning approaches under low-resource settings.

\section{Query Tool Details}
\label{sec:tools}
We introduce two atomic retrieval tools to support low-budget long-horizon omnimodal QA: an image retrieval tool and an audio retrieval tool. These tools allow an OmniLLM to convert a natural language question into structured retrieval actions and iteratively collect fine-grained evidence from the image--audio banks. Each tool call returns a compact evidence package, which can be consumed by the model for subsequent multi-turn reasoning without loading the full raw audio--video stream.

\textbf{Image Retrieval Tool.} This tool is used to retrieve time-localized visual evidence from the image bank that is most relevant to the query. It takes \texttt{query}, \texttt{top\_k}, and \texttt{video\_id} as arguments. Given the query, the tool encodes it with a CLIP text encoder and performs nearest-neighbor search over a FAISS index~\cite{douze2025faiss} built from sampled video frames. To reduce redundancy and improve temporal coverage, we further aggregate frame-level hits into segment-level evidence by grouping results with the same \texttt{segment\_id} and keeping only the best-matching frame per segment. The returned observation is the top-$k$ ranked segments, where each segment contains \texttt{segment\_id}, timestamps (\texttt{t0}/\texttt{t1}), corresponding \texttt{frame}/\texttt{audio} paths, and a similarity score. For example, for the query ``where am I in relation to the red car?'', the tool returns several high-scoring segments whose timestamps localize the relevant frames for visual grounding.

\textbf{Audio Retrieval Tool.} This tool is used to retrieve speech-related evidence from an ASR-indexed audio bank. It takes \texttt{query}, \texttt{top\_k}, and \texttt{video\_id} as arguments. The tool embeds the query text and performs FAISS retrieval over the audio index constructed from time-stamped ASR segments, returning the top-$k$ most relevant audio segments. The observation includes \texttt{segment\_id}, timestamps (\texttt{t0}/\texttt{t1}), the \texttt{audio} path, the associated ASR \texttt{transcript}, and a similarity score. This transcript-based return format enables the OmniLLM to read the retrieved evidence directly, avoiding expensive loading of long audio streams. For instance, for a question such as ``when I say `Woohoo', where am I?'', the agent can first invoke audio retrieval to locate the segment and then optionally switch to image retrieval to ground the spatial relation visually.

\section{Data Processing Details}
\textbf{Video temporal downsampling (budgeted compression).}
To enable long-horizon reasoning under a limited input budget, we first apply a temporal compression step to each raw video.
The preprocessing pipeline measures the original video length and rescales the playback timeline so that the entire stream is mapped into a fixed shorter duration.
Both the visual sequence and the accompanying audio track are adjusted consistently, ensuring that the compressed video preserves global temporal structure while fitting within a constrained context window.
This produces a lightweight surrogate input that can be processed efficiently by the model during long-horizon question answering.

\textbf{Frame sampling (building the image bank).}
To recover fine-grained visual details lost during compression, we additionally construct an external image evidence bank.
The video is partitioned into uniform temporal segments, and representative frames are sampled at regular intervals across the full timeline.
Each sampled frame is stored together with its associated timestamp metadata, forming a time-indexed visual repository.
This design allows the agent to later retrieve short, localized visual evidence on demand, rather than encoding the full video densely.

\textbf{ASR extraction and audio bank indexing.}
In parallel, we build an audio evidence bank by extracting short audio snippets aligned with the same temporal segmentation.
Each snippet is transcribed into text using an automatic speech recognition module, yielding a sequence of time-stamped transcripts.
These transcripts are then embedded into a searchable representation space and organized into an indexed repository.
As a result, the agent can efficiently retrieve salient speech or sound evidence through natural-language queries, providing fine-grained audio grounding for multimodal reasoning.

\section{Baseline Details}
\label{sec:baselines}
We evaluate our approach against a set of strong multi-modal large language models (LLMs) as baselines, covering both closed-source and open-source omni-modal systems. These baselines represent state-of-the-art generalist models capable of reasoning over vision and language inputs without additional retrieval, agentic planning, or reinforcement learning enhancements.

\textbf{GPT-5}~\cite{ma2026safety} is the latest flagship multimodal model in the GPT series released by OpenAI, designed to balance intelligent reasoning and interactive communication across text and visual inputs. It introduces configurable reasoning effort modes and improved instruction following, enabling more reliable responses on both simple and complex tasks compared to its predecessors. GPT-5 supports multimodal understanding of text and images, and its architecture has been optimized for faster reasoning and enhanced contextual comprehension, making it a representative closed-source baseline for complex multi-modal question answering.

\textbf{Gemini 2.0-Flash}~\cite{team2023gemini} is a member of the Gemini family of highly capable multimodal large language models developed by Google DeepMind and Google Research, designed to support natively integrated understanding of text, images, audio, and video inputs. Compared to its predecessors in the Gemini series, Gemini 2.0-Flash offers enhanced multimodal reasoning quality while maintaining efficient inference speed, facilitating deeper cross-modal interaction and step-wise problem solving across complex inputs. 

\textbf{Gemini 2.5-Flash}~\cite{comanici2025gemini} is an advanced multimodal large language model in the Gemini 2.X family developed by Google DeepMind and Google Research, designed to push the frontier of reasoning, long-context processing, and native multimodal understanding. Building upon the foundation of earlier Gemini models, Gemini 2.5-Flash incorporates hybrid reasoning capabilities that balance high reasoning quality with efficient inference, allowing configurable "thinking" and speed-optimized operation. It supports integrated comprehension of text, images, audio, and video inputs and can process long contexts with rich cross-modal interactions, enabling it to handle complex multi-step tasks across diverse modalities. 

\textbf{Qwen2.5-Omni}~\cite{xu2025qwen2} is an end-to-end multimodal large language model developed by the Qwen team, designed to perceive and jointly reason over diverse modalities including text, images, audio, and video. It introduces architectural innovations such as block-wise processing for streaming multimodal inputs and a novel time-aligned multimodal positional embedding to synchronize audio and video representations, which facilitate coherent cross-modal understanding and long-context processing. Furthermore, Qwen2.5-Omni adopts a unified Thinker-Talker framework that decouples text generation and speech synthesis while enabling simultaneous outputs, yielding robust performance on multimodal benchmarks.

\textbf{Qwen3-Omni}~\cite{xu2025qwen3omnitechnicalreport} is a unified end-to-end multimodal foundation model developed by the Qwen team, designed to support integrated understanding and generation across text, images, audio, and video without sacrificing performance in any individual modality. For the first time, Qwen3-Omni achieves state-of-the-art results on a wide range of benchmarks in all four modalities, matching or exceeding the performance of same-sized single-modal counterparts while often surpassing strong closed-source models on audio and audio-visual tasks. It employs a Thinker–Talker mixture-of-experts architecture that decouples reasoning and real-time speech synthesis, enabling fluent multilingual interaction and low-latency responses, and supports long-context comprehension with rich cross-modal reasoning.

\section{Hyperparameter Settings}
\label{sec:hyperparameters}
The hyperparameter configuration of OmniRAG-Agent is divided into two main stages: reinforcement learning–based training and multi-turn retrieval-augmented inference. All hyperparameters are designed to support long-context~\cite{zhao2024longrag}, low-resource audio–visual question answering under a constrained interaction budget, with explicit limits on sequence length, retrieval scope, and multi-turn reasoning depth to ensure stable optimization and efficient inference.

\textbf{Training Hyperparameters.} During training, we adopt a GRPO-based optimization scheme with a batch size of 4. The maximum prompt length and response length are set to 28{,}672 and 4{,}096 tokens. Each episode is limited to at most 20 environment steps, with a history length of 20 turns, and 5 rollouts are sampled per training instance. For retrieval during training, the top-$k$ value is set to 3, with up to 3 image segments and 3 audio segments returned per query. The actor learning rate is set to $5\times10^{-7}$, with a PPO mini-batch size of 4 and a micro-batch size of 1 per GPU. A KL loss coefficient of 0.001 is used, while the entropy coefficient is set to 0. Training is conducted for a single epoch. To support long-context modeling, the maximum model length and maximum number of batched tokens are both set to 32{,}768.

\textbf{Inference Hyperparameters.} At inference time, we evaluate the multi-turn RAG-based agent under a constrained reasoning budget. The maximum number of interaction turns is set to 20. For each retrieval step, the top-$k$ value is increased to 5, and up to 3 evidence segments are attached to the model input to balance coverage and efficiency. The decoding temperature~\cite{sheng2025hybridflow,zhang2025agentrl} is set to 0.2, and the maximum number of newly generated tokens is limited to 512 to ensure stable and concise responses during long-horizon reasoning.

\section{Future Directions}
OmniRAG-Agent opens up several promising directions for future research and development. Below we highlight a few avenues that could further improve long-horizon audio--video question answering under low-resource constraints, expand applicability, and address current limitations.

\textbf{Stronger End-to-End Optimization for Tool Use and Evidence Selection.}
While we adopt GRPO to improve tool use and answer quality, future work could explore alternative preference-based and offline RL objectives (e.g., DPO-style optimization) to make training more stable and data-efficient. A key challenge is building high-quality preference pairs and trajectory-level supervision for multi-turn retrieval, including cases where partial evidence is correct but the final answer fails. Promising directions include automatic preference construction from evaluator models, uncertainty-aware labeling, and lightweight human-in-the-loop verification for hard examples.

\textbf{More Fine-Grained Multimodal Retrieval and Temporal Grounding.}
Our retrieval environment currently selects evidence via image/audio banks and top-$k$ search. Future work can improve temporal grounding by incorporating hierarchical retrieval (coarse-to-fine), segment proposal generation, redundancy control, and explicit coverage objectives across long timelines. Another direction is learning retrieval representations jointly with the agent (e.g., query-aware embedding adaptation) and introducing cross-modal constraints to better align audio events with visually corresponding segments, which is essential for audio--visual event alignment and causal reasoning.

\textbf{Richer Agent Memory, Planning, and Verification Loops.}
Multi-turn interaction helps gather scattered clues, but longer horizons require better memory management and verification. Future iterations could incorporate structured memory (e.g., timeline graphs, entity/event tables), dynamic summarization with correctness guarantees, and explicit self-verification routines that trigger follow-up retrieval when evidence is insufficient or contradictory. Additionally, planning could be improved by learning policies that allocate retrieval budget adaptively based on question difficulty, current uncertainty, and marginal value of additional evidence.

\textbf{Deployment-Oriented Efficiency and Robustness in Real-World Scenarios.}
To move toward real-world deployment, it is important to study robustness under noisy ASR, imperfect frame sampling, domain shift, and long-context failure modes. Future work could investigate error-aware retrieval (robust to transcript noise), lightweight compression of evidence, and calibrated stopping criteria that reduce unnecessary tool calls. Another practical direction is integrating privacy-preserving or on-device retrieval components and evaluating the system under strict latency and memory budgets.

\textbf{Extensions to Multilingual, Multimodal, and Multi-Agent Settings.}
Finally, OmniRAG-Agent can be extended to multilingual audio--video QA by improving cross-lingual retrieval and enabling mixed-language reasoning over speech and subtitles. Beyond single-agent reasoning, multi-agent collaboration (e.g., separate ``audio specialist'' and ``vision specialist'' agents) may improve exploration and verification for complex questions, while multimodal toolchains (e.g., OCR, speaker diarization, event detection) could further enrich the evidence space and broaden the framework's applicability.

%%%%%%%%%%%%%%%%%%%%%%%%%%%%%%%%%%%%%%%%%%%%%%%%%%%%%%%%%%%%%%%%%%%%%%%%%%%%%%%
%%%%%%%%%%%%%%%%%%%%%%%%%%%%%%%%%%%%%%%%%%%%%%%%%%%%%%%%%%%%%%%%%%%%%%%%%%%%%%%

\end{document}